\documentclass{article}
\usepackage{ijcai20}

\usepackage{times}

\usepackage{soul}
\usepackage{url}
\usepackage[hidelinks]{hyperref}
\usepackage[utf8]{inputenc}
\usepackage[small]{caption}
\usepackage{graphicx}
\usepackage{amsfonts,amsmath,bm}
\usepackage{booktabs}
\newcommand{\citep}[1]{\cite{#1}}
\newcommand{\citet}[1]{\cite{#1}}

\usepackage{subfig}
\usepackage{multirow}
\usepackage{makecell}
\usepackage{wrapfig}
\usepackage{enumitem}

\usepackage{amsthm}

\usepackage{pgfplots,pgfplotstable}
\pgfplotsset{compat=1.15}
\pgfplotsset{every tick label/.append style={font=\scriptsize}}
\pgfplotsset{every label/.append style={font=\scriptsize}}

\usepackage{algorithm,algpseudocode}

\newcommand{\HSIC}{\text{HSIC}}
\newcommand{\empHSIC}{\widehat\HSIC}
\newcommand{\E}{\mathbb{E}}

\DeclareMathOperator*{\minimize}{minimize}

\pgfplotsset{
    /pgfplots/xlabel near ticks/.style={
        /pgfplots/every axis x label/.style={
        at={(ticklabel cs:0.5)},anchor=near ticklabel
        }
    },
    /pgfplots/ylabel near ticks/.style={
        /pgfplots/every axis y label/.style={
        at={(ticklabel cs:0.5)},rotate=90,anchor=near ticklabel}
    }
}

\theoremstyle{plain}
\newtheorem{prop}{Proposition}
\newtheorem{exmp}{Example}
\newtheorem{defn}{Definition}
\newtheorem{asmp}{Assumption}
\theoremstyle{remark}

\title{Knowledge-Based Regularization in Generative Modeling}

\author{
Naoya Takeishi$^1$\footnote{Contact Author}\And
Yoshinobu Kawahara$^{2,1}$
\\
\affiliations
$^1$RIKEN Center for Advanced Intelligence Project\\
$^2$Institute of Mathematics for Industry, Kyushu University\\
\emails
naoya.takeishi@riken.jp, \
kawahara@imi.kyushu-u.ac.jp
}

\begin{document}

\maketitle

\begin{abstract}
Prior domain knowledge can greatly help to learn generative models. However, it is often too costly to hard-code prior knowledge as a specific model architecture, so we often have to use general-purpose models. In this paper, we propose a method to incorporate prior knowledge of feature relations into the learning of general-purpose generative models. To this end, we formulate a regularizer that makes the marginals of a generative model to follow prescribed relative dependence of features. It can be incorporated into off-the-shelf learning methods of many generative models, including variational autoencoders and generative adversarial networks, as its gradients can be computed using standard backpropagation techniques. We show the effectiveness of the proposed method with experiments on multiple types of datasets and generative models.
\end{abstract}


\section{Introduction}
\label{intro}

Generative modeling plays a key role in many scientific and engineering applications.
It has often been discussed with the notion of graphical models \citep{PGMBook}, which are built upon knowledge of conditional independence of features.
We have also seen the advances in neural network techniques for generative modeling, such as variational autoencoders (VAEs) \citep{Kingma14} and generative adversarial networks (GANs) \citep{Goodfellow14}.

For efficient generative modeling, a model should be designed following \emph{prior knowledge} of target phenomena.
If one knows the conditional independence of features, it can be encoded as a graphical model.
If there are some insights on the physics of the phenomena, a model can be built based on a known form of differential equations.
Otherwise, a special neural network architecture may be created in accordance with the knowledge.
However, although some tools (e.g., \cite{Koller97}) have been suggested, it is often labor-intensive and possibly even infeasible to design a generative model meticulously so that the prior knowledge specific to each problem instance is hard-coded in the model.

On the other hand, we may employ general-purpose generative models, such as kernel-based models and neural networks with common architectures (multi-layer perceptrons, convolutional nets, and recurrent nets, etc.).
Meanwhile, a general-purpose model, if used as is, is less efficient than specially-designed ones in terms of sample complexity because of large hypothesis space.
Therefore, we want to incorporate as much prior knowledge as possible into a model, somehow avoiding the direct model design.

A promising way to incorporate prior knowledge into general-purpose models is via regularization.
For example, structured sparsity regularization \citep{Huang11} is known as a method for regularizing linear models with a rigorous theoretical background.
In a related context, posterior regularization \citep{Ganchev10,Zhu14} has been discussed as a methodology to impose expectation constraints on learned distributions.
However, the applicability of the existing regularization methods is still limited in terms of target types of prior knowledge and base models.

In generative modeling, we often have prior knowledge of \emph{relationship between features}.
For instance, consider generative modeling of sensor data, which is useful for tasks such as control and fault detection.
In many instances, we know \emph{to some extent} how units and sensors in a plant are connected (Figure~\ref{fig:knowledge:partial}), from which the dependence of a part of features (here, sensor readings) can be anticipated.
In the example of Figure~\ref{fig:knowledge:partial}, features $x_1$ and $x_2$ would be more or equally dependent than $x_1$ and $x_3$ because of external disturbances in the processes between the units.
Another example is when we know the pairwise relationship of features (Figure~\ref{fig:knowledge:sideinfo}), which can be derived from some side information.

The point is that we rarely know the full data structure, which is necessary for building a graphical model.
Instead, we know \emph{a part of} the data structure or pairwise relationship of \emph{some} features, and there remain unknown parts that should be modeled using general-purpose models.
Most existing methods cannot deal with such partial knowledge straightforwardly.
In this paper, we propose a regularizer to incorporate such knowledge into general-purpose generative models based on the idea that statistical dependence of features can be (partly) anticipated from the relationship of features.
The use of this type of knowledge has been actively discussed in several contexts of machine learning, but there have been surprisingly few studies in the context of generative modeling.

The proposed regularizer is defined using a kernel-based criterion of dependence \citep{Gretton05}, which is advantageous because its gradients can be computed using standard backpropagation without additional iterative procedures.
Consequently, it can be incorporated into off-the-shelf gradient-based learning methods of many general-purpose generative models, such as latent variable models (e.g., factor analysis, topic models, etc.), VAEs, and GANs.
We conducted experiments using multiple datasets and generative models, and the results showcase that a model regularized using prior knowledge of feature relations achieves better generalization.
The proposed method can provide a trade-off between performance and workload for model design.



\section{Background}
\label{back}

In this section, we briefly review two technical building blocks: generative modeling and dependence criteria.

\subsection{Generative Modeling}

We use the term ``learning generative models'' in the sense that we are to learn $p_\theta(\bm{x})$ explicitly or implicitly.
Here, $\bm{x}$ denotes the observed variable, and $\theta$ is the set of parameters to be estimated.
Many popular generative models are built as Bayesian networks and Markov random fields \cite{PGMBook}.
Also, the advances in deep learning technique include VAEs \citep{Kingma14}, GANs (e.g., \citep{Goodfellow14}), autoregressive models (e.g., \cite{vandenOord16}), and normalizing flows (e.g., \cite{Dinh18}).

The learning strategies for generative models are usually based on minimization of some loss function $L(\theta)$:
\begin{equation}\label{eq:optim}
  \minimize_\theta ~ L(\theta).
\end{equation}
A typical loss function is the negative log-likelihood or its approximation (e.g., ELBO in variational Bayes) \citep{PGMBook,Kingma14}.
Another class of loss functions is those designed to perform the comparison of model- and data-distributions, which has been studied recently often in the context of learning implicit generative models that have no explicit expression of likelihood.
For example, GANs \citep{Goodfellow14} are learned via a two-player game between a discriminator and a generator.

\begin{figure}[t]
    \centering
    \begin{minipage}[t]{4.2cm}
        \vspace*{0mm}
        \subfloat[]{
            \centering
            \begin{tikzpicture}
                \node [fill=gray!40,circle,inner sep=2pt,minimum size=.45cm,draw,label=below:{\scriptsize Sensor 1}] (s1) at (0.46,0.65) {\scriptsize $x_1$};
                \node [fill=gray!40,circle,inner sep=2pt,minimum size=.45cm,draw,label=below:{\scriptsize Sensor 2}] (s2) at (1.74,0.65) {\scriptsize $x_2$};
                \node [fill=gray!40,circle,inner sep=2pt,minimum size=.45cm,draw,label=above:{\scriptsize Sensor 3}] (s3) at (2.9,2.0) {\scriptsize $x_3$};
                \node [fill=white,minimum size=.3cm,draw,label={[label distance=0mm]180:{\scriptsize Unit 1}}] (u1) at (.46,1.3) {\scriptsize $z_1$};
                \node [fill=white,minimum size=.3cm,draw,label={[label distance=0mm]0:{\scriptsize Unit 2}}] (u2) at (1.74,1.3) {\scriptsize $z_2$};
                \node [fill=white,minimum size=.3cm,draw,label={[label distance=0mm]180:{\scriptsize Unit 3}}] (u3) at (1.74,2.0) {\scriptsize $z_3$};
                \draw[->] (u1) -- (s1);
                \draw[->] (u2) -- (s2);
                \draw[->] (u3) -- (s3);
                \draw[thick,->] (u1) -- (u2);
                \draw[thick,->] (u2) -- (u3);
            \end{tikzpicture}
            \label{fig:knowledge:partial}
        }
    \end{minipage}
    \hfill
    \begin{minipage}[t]{3.9cm}
        \vspace*{0mm}
        \subfloat[]{
            \centering
            \begin{tikzpicture}
                \node [fill=gray!40,circle,inner sep=2pt,outer sep=2pt,minimum size=.45cm,draw,label=right:{\scriptsize Feature 1}] (w1) at (2.1,0.3) {\scriptsize $x_1$};
                \node [fill=gray!40,circle,inner sep=2pt,outer sep=2pt,minimum size=.45cm,draw,label=right:{\scriptsize Feature 2}] (w2) at (2.1,1.0) {\scriptsize $x_2$};
                \node [fill=gray!40,circle,inner sep=2pt,outer sep=2pt,minimum size=.45cm,draw,label=right:{\scriptsize Feature 3}] (w3) at (2.1,1.7) {\scriptsize $x_3$};
                \node [fill=gray!40,circle,inner sep=2pt,outer sep=2pt,minimum size=.45cm,draw,label=right:{\scriptsize Feature 4}] (w4) at (2.1,2.4) {\scriptsize $x_4$};
                \node [text width=1.2cm,align=right,inner sep=2pt] (r1) at (0.6,0.65) {\scriptsize similar};
                \node [text width=1.2cm,align=right,inner sep=2pt] (r2) at (0.6,1.35) {\scriptsize dissimilar};
                \node [text width=1.2cm,align=right,inner sep=2pt] (r3) at (0.6,2.05) {\scriptsize similar};
                \draw[] (w1) -- (r1.east);
                \draw[] (w2) -- (r1.east);
                \draw[] (w2) -- (r2.east);
                \draw[] (w3) -- (r2.east);
                \draw[] (w3) -- (r3.east);
                \draw[] (w4) -- (r3.east);
            \end{tikzpicture}
            \label{fig:knowledge:sideinfo}
        }
    \end{minipage}
    \caption{Examples of knowledge of feature dependence.}
    \label{fig:knowledge}
\end{figure}
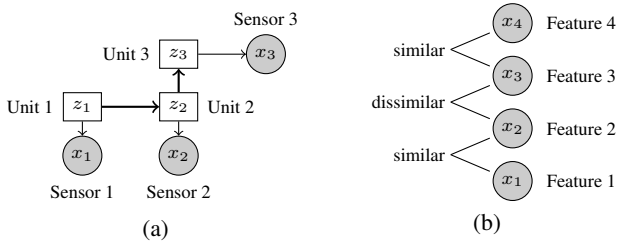

\subsection{Dependence Criteria}

Among several measures of statistical dependence of random variables, we adopt a kernel-based method, namely Hilbert--Schmidt independence criterion (HSIC) \citep{Gretton05}.
The advantage of HSIC is discussed later in Section~\ref{regularizer}.
Below we review the basic concepts.

HSIC is defined and computed as follows \citep{Gretton05}.
Let $p_{xy}$ be a joint measure over $(\mathcal{X}\times\mathcal{Y}, \Gamma\times\Lambda)$, where $\mathcal{X}$ and $\mathcal{Y}$ are separable spaces, $\Gamma$ and $\Lambda$ are Borel sets on $\mathcal{X}$ and $\mathcal{Y}$, and $(\mathcal{X},\Gamma)$ and $(\mathcal{Y},\Lambda)$ are furnished with probability measure $p_x$ and $p_y$, respectively.
Given reproducing kernel Hilbert spaces (RKHSs) $\mathcal{F}$ and $\mathcal{G}$ on $\mathcal{X}$ and $\mathcal{Y}$, respectively, HSIC is defined as the squared Hilbert--Schmidt norm of a cross-covariance operator $C_{xy}$, i.e., $\HSIC(\mathcal{F},\mathcal{G},p_{xy}) := \Vert C_{xy} \Vert_\mathrm{HS}^2$.
When bounded kernels $k$ and $l$ are uniquely associated with the RKHSs, $\mathcal{F}$ and $\mathcal{G}$, respectively,
\begin{multline}\label{eq:hsic_pop}
    \HSIC(\mathcal{F},\mathcal{G},p_{xy}) =
    \E_{x,x',y,y'}[k(x,x')l(y,y')]
    \\
    + \E_{x,x'}[k(x,x')] \E_{yy'}[l(y,y')]
    \\
    - 2 \E_{x,y} \big[ \E_{x'}[k(x,x')] \E_{y'}[l(y,y')] \big] \geq 0.
\end{multline}
HSIC works as a dependence measure because $\HSIC=0$ if and only if two random variables are statistically independent (Theorem~4, \cite{Gretton05}).

HSIC can be empirically estimated using a dataset $\mathcal{D}=\{(x_1,y_1),\dots,(x_m,y_m)\}$ from \eqref{eq:hsic_pop}.
Gretton \emph{et al.} \citep{Gretton05} presented a biased estimator with $O(m^{-1})$ bias, and Song \emph{et al.} \citep{Song12} suggested an unbiased estimator.
In what follows, we denote an empirical estimation of $\HSIC(\mathcal{F},\mathcal{G},p_{xy})$ computed using a dataset $\mathcal{D}$ by $\empHSIC_{x,y}(\mathcal{D})$.
While the computation of empirical HSICs requires $O(m^2)$ operations, it may be sped up by methods such as random Fourier features \citep{Zhang18}.


As later discussed in Section~\ref{motivation}, we particularly exploit \emph{relative} dependence of the features.
Bounliphone \emph{et al.} \citep{Bounliphone15} discussed the use of HSIC for a test of relative dependence.
Here, we introduce $\mathcal{H}$ as a separable RKHS on another separable space $\mathcal{Z}$.
Then, the test of relative dependence (i.e., which pair is more dependent, $(x,y)$ or $(x,z)$?) is formulated with the null and alternative hypotheses:
\begin{equation*}
    H_0\colon \ \rho_{x,y,z} \leq 0
    \quad\text{and}\quad
    H_1\colon \ \rho_{x,y,z} > 0,
\end{equation*}
where $\rho_{x,y,z} := \HSIC(\mathcal{F},\mathcal{G},p_{xy}) - \HSIC(\mathcal{F},\mathcal{H},p_{xz})$.
For this test, they define a statistic
\begin{equation}\label{eq:statistic}
    \hat\rho_{x,y,z} := \empHSIC_{x,y}(\mathcal{D})-\empHSIC_{x,z}(\mathcal{D}).
\end{equation}
From the asymptotic distribution of empirical HSIC, it is known \citep{Bounliphone15} that a conservative estimation of $p$-value of the test with $\hat\rho_{x,y,z}$ is obtained as
\begin{equation}\label{eq:pvalue}
    p \leq 1 - \Phi \big( \hat\rho_{x,y,z} (\sigma^2_{xy}+\sigma^2_{xz}-2\sigma_{xyxz})^{-1/2} \big),
\end{equation}
where $\Phi$ is the CDF of the standard normal, and $\sigma$'s denote the standard deviations of the asymptotic distribution of HSIC.
See \citep{Bounliphone15} for the definition of $\sigma^2_{xy}$, $\sigma^2_{xz}$, and $\sigma_{xyxz}$, which can also be estimated empirically.
Here, \eqref{eq:pvalue} means that under the null hypothesis $H_0$, the probability that $\rho_{x,y,z}$ is greater than or equal to $\hat\rho_{x,y,z}$ is bounded by this value.
In the proposed method, we exploit this fact to formulate a regularization term.


\section{Proposed Method}
\label{method}

First, we manifest the type of generative models to which the proposed method applies.
Then, we define what we expect to have as prior knowledge.
Finally, we present the proposed regularizer and give its interpretation.

\subsection{Target Type of Generative Models}
\label{target}

Our regularization method is agnostic of the original loss function of generative modeling and the parametrization, as long as the following (informal) conditions are satisfied:
\begin{asmp}\label{asmp:1}
    Samples from $p_\theta(\bm{x})$, namely $\hat{\bm{x}}$, can be drawn with an admissible computational cost.
\end{asmp}
\begin{asmp}\label{asmp:2}
    Gradients $\nabla_\theta \mathbb{E}_{p_\theta(\bm{x})} \big[ f(\bm{x}) \big]$ can be (approximately) computed with an admissible computational cost.
\end{asmp}
While Assumption~\ref{asmp:1} is satisfied in most generative models, Assumption~\ref{asmp:2} is a little less obvious.
We note that the efficient computation of $\nabla_\theta \mathbb{E}_{p_\theta(\bm{x})} \big[ f(\bm{x}) \big]$ is often inherently easy or facilitated with techniques such as the reparameterization trick \citep{Kingma14} or its variants, and thus these assumptions are satisfied by many popular methods such as factor analysis, VAEs, and GANs.

\subsection{Target Type of Prior Knowledge}
\label{motivation}

We suppose that we know the (partial) relationship of features that can be encoded as plausible relative dependence of the features.
Such knowledge is frequently available in many practices of generative modeling, and in fact, the use of such knowledge has also been considered in other contexts of machine learning (see Section~\ref{useofknowledge}).
Below we introduce several motivating examples and then give a formal definition.

\subsubsection{Motivating Examples}

The first motivating example is when we know (a part of) the data-generating process.
Suppose to learn a generative model of sensor data of an industrial plant.
We usually know how units and sensors in the plant are connected, which is an important source of prior knowledge.
In Figure~\ref{fig:knowledge:partial}, suppose that Unit~$i$ has internal state $z_i$ for $i=1,\dots,3$.
Also, suppose there are relations $z_2=h_{12}(z_1,\omega_1)$ and $z_3=h_2(z_2,\omega_2)$, where $\omega$'s are random noises and $h$'s are functions of physical processes.
Then, $(p(z_1),p(z_2))$ would be more statistically dependent than $(p(z_1),p(z_3))$ because of the presence of $\omega_2$.
If the sensor readings, $\{x_i\}$, are determined by $x_i=g(z_i)$ with observation functions $\{g_i\}$, then, $(p(x_1),p(x_2))$ would be more statistically dependent than $(p(x_1),p(x_3))$ analogously to $z$.
Figure~\ref{fig:knowledge:partial} is revisited in Example~\ref{exmp:1}.
This type of prior knowledge is often available also for physical and biological phenomena.

Another type of example is when we know pairwise similarity or dissimilarity of features (Figure~\ref{fig:knowledge:sideinfo}).
For example, we may estimate similarities of distributed sensors from their locations.
Also, we may anticipate similarities of words using ontology or word embeddings.
Moreover, we may know the similarities of molecules from their descriptions in the chemical compound analysis.
We can anticipate relative dependence from such information, i.e., directly similar feature pairs are more dependent than dissimilar feature pairs.
Figure~\ref{fig:knowledge:sideinfo} is revisited in Example~\ref{exmp:1}.
Here, we do not have to know every pairwise relation; knowledge on \emph{some} feature pairs is sufficient to anticipate the relative dependence.


\subsubsection{Definition of Prior Knowledge}

As the exact degree of feature dependence can hardly be described precisely from prior knowledge, we use the dependence of a pair of features \emph{relative to another pair}.
This idea is formally written as follows:
\begin{defn}[Knowledge of feature dependence]\label{def:knowledge}
    Suppose that the observed random variable, $\bm{x}$, is a tuple of $d$ random variables, i.e., $\bm{x}=(x_1,\dots,x_d)$.
    Knowledge of feature dependence is described as a set of triples:
    \begin{equation}\label{eq:def}
        \mathcal{K} := \big\{ (J_s^\text{ref},J_s^+,J_s^-) \mid s=1,\dots,\vert\mathcal{K}\vert \big\},
    \end{equation}
    where $J^\text{ref},J^+,J^- \subset \{1,\dots,d\}$ are index sets.
    The semantics are as follows.
    Let $J=\{i_1,\dots,i_{\vert J \vert}\}$, and let $\bm{x}_J$ be the subtuple of $\bm{x}$ by $J$, i.e., $\bm{x}_J = (x_{i_1},\dots,x_{i_{\vert J \vert}})$.
    Then, triple $(J_s^\text{ref},J_s^+,J_s^-)$ encodes the following piece of knowledge: ``$\bm{x}_{J_s^\text{ref}}$ is more dependent on $\bm{x}_{J_s^+}$ than on $\bm{x}_{J_s^-}$.''
\end{defn}
\begin{exmp}\label{exmp:1}
    $\mathcal{K}=\{(\{1\}, \{2\}, \{3\})\}$ in Figure~\ref{fig:knowledge:partial} and $\mathcal{K}=\{(\{1\}, \{2\}, \{3,4\})\}$ in Figure~\ref{fig:knowledge:sideinfo}.
\end{exmp}
More examples of $\mathcal{K}$ are in Section~\ref{expt}.

\subsection{Proposed Regularization Method}
\label{regularizer}

We define the proposed regularization method and give its interpretation as a probabilistic penalty method.

\subsubsection{Definition of Regularizer}

We want to force a generative model $p_\theta(\bm{x})$ to follow the relations encoded in $\mathcal{K}$.
To this end, the order of HSIC between the marginals of $p_\theta(\bm{x})$ should be as consistent as possible to the relations in $\mathcal{K}$.
More concretely, the HSIC of $(p_\theta(\bm{x}_{J^\text{ref},s}),p_\theta(\bm{x}_{J^+,s}))$ should be larger than that of $(p_\theta(\bm{x}_{J^\text{ref},s}),p_\theta(\bm{x}_{J^-,s}))$, for $s=1,\dots,\vert\mathcal{K}\vert$.

Because directly imposing constraints on the true HSIC of the marginals is intractable, we resort to a penalty method using empirical HSIC.
Concretely, we add a term that penalizes the violation of the knowledge in $\mathcal{K}$ as follows.
\begin{defn}[Knowledge-based regularization]\label{def:main}
    Let $L(\theta)$ be the original loss function in \eqref{eq:optim}.
    The learning problem regularized using $\mathcal{K}$ in \eqref{eq:def} is posed as
    \begin{equation}\label{eq:optim_reg}
        \minimize_{\theta} ~ L(\theta) + \lambda R_\mathcal{K}(\theta),
    \end{equation}
    where $\lambda \geq 0$ is a regularization hyperparameter, and
    \begin{equation}\label{eq:reg}\begin{gathered}
        R_\mathcal{K}(\theta) := \frac1{\vert \mathcal{K} \vert} \sum_{s=1}^{\vert \mathcal{K} \vert} R_{\mathcal{K},s}(\theta),
        \\
        R_{\mathcal{K},s}(\theta) := \max(0, \  \nu_\alpha - \hat\rho_s(\theta)).
    \end{gathered}\end{equation}
    Here, $\nu_\alpha \geq 0$ is another hyperparameter.
    $\hat{\rho}_s(\theta)$ is the empirical estimation of the relative HSIC corresponding to the $s$-th triple in $\mathcal{K}$, that is,
    \begin{equation}\label{eq:rho:emp}
        \hat{\rho}_s(\theta) := \empHSIC_{x_{J^\text{ref}_s}, x_{J^+_s}}(\hat{\mathcal{D}}) - \empHSIC_{x_{J^\text{ref}_s}, x_{J^-_s}}(\hat{\mathcal{D}}),
    \end{equation}
    where $\hat{\mathcal{D}}$ denotes a set of samples drawn from $p_\theta(\bm{x})$.
\end{defn}


\subsubsection{Interpretation}

By imposing the regularizer $R_\mathcal{K}$, we expect that the hypothesis space to be explored is made smaller implicitly, which would result in better generalization capability.
This is also supported by the following interpretation of the regularizer.

While $R_\mathcal{K}(\theta)$ depends on the \emph{empirical} HSIC values, making (each summand of) $R_\mathcal{K}(\theta)$ small can be interpreted as imposing ``probabilistic'' constraints on the order of the \emph{true} HSIC values via a penalty method.
Now let $\rho_s$ denote the true value of $\hat{\rho}_s$ (we omit argument $\theta$ for simplicity), i.e.,
\begin{equation}
    \rho_s := \HSIC(\mathcal{F}_{J^\text{ref}_s}, \mathcal{F}_{J^+_s}) - \HSIC(\mathcal{F}_{J^\text{ref}_s}, \mathcal{F}_{J^-_s}),
\end{equation}
where $\mathcal{F}_J$ denotes a separable RKHS on the space of $\bm{x}_J$.
Moreover, consider a test with hypotheses:
\begin{equation}\label{eq:hypothesis}
    H_{s,0}\colon \ \rho_s \leq 0
    \quad\text{and}\quad
    H_{s,1}\colon \ \rho_s > 0,
\end{equation}
for $s=1,\dots,\vert\mathcal{K}\vert$.
On the test with $H_{s,0}$ and $H_{s,1}$, the following proposition holds:
\begin{prop}\label{prop:main}
    If $R_{\mathcal{K},s}=0$ is achieved, then the null hypothesis $H_{s,0}$ can be rejected in favor of the alternative $H_{s,1}$ with $p$-value $p_s$ upper bounded by
    \begin{equation}\label{eq:prop:main}
        p_s \leq 1 - \Phi(\nu_\alpha / \tau_s), \quad \exists \tau_s>0.
    \end{equation}
\end{prop}
\begin{proof}
    From \eqref{eq:pvalue}, we have $p_s \leq 1 - \Phi(\hat{\rho}_s / \tau_s)$, where
    \begin{equation*}
        \tau_s = (\sigma^2_{J^\text{ref}_s J^+_s} + \sigma^2_{J^\text{ref}_s J^-_s} - 2 \sigma_{J^\text{ref}_s J^+_s J^\text{ref}_s J^-_s})^{1/2} > 0.
    \end{equation*}
    Also, when $R_{\mathcal{K},s}=0$, we have $\hat{\rho}_s \geq \nu_\alpha$.
    Consequently, as $\Phi$ is monotonically increasing, we have \eqref{eq:prop:main}.
\end{proof}

\subsubsection{Choice of Hyperparameters}

The proposed regularizer has two hyperparameters, $\lambda$ and $\nu_\alpha$.
While $\lambda$ is a standard regularization parameter that balances the loss and the regularizer, $\nu_\alpha$ can be interpreted as follows.
Now let $\alpha := 1-\Phi(\nu_\alpha/\tau_s)$ (i.e., the right-hand side of \eqref{eq:prop:main}), which corresponds to the required significance of our probabilistic constraints.
We can determine $0<\alpha<1$ based on the plausibility of the prior knowledge, such as $\alpha=0.05$ or $\alpha=0.1$ (a smaller $\alpha$ requires higher significance).
Also, $\tau_s$ can be estimated as it is defined with the variances of the asymptotic distribution of HSIC \citep{Bounliphone15}.
Hence, as $\nu_\alpha = \tau_s \Phi^{-1}(1 - \alpha)$, we can roughly determine the value of $\nu_\alpha$ that corresponds to a specific significance, $\alpha$.
The need to tune $\nu_\alpha$ (or $\alpha$) remains, but the above interpretation is useful to determine the search range for $\nu_\alpha$.

\subsubsection{Optimization Method}

From Assumptions~\ref{asmp:1} and \ref{asmp:2}, the gradient of $\hat\rho_s$ with regard to $\theta$ can be computed with the backpropagation via the samples from $p_\theta(\bm{x})$, which are used to compute empirical HSIC.
Hence, if the solution of the original optimization, \eqref{eq:optim}, is obtained via a gradient-based method, the regularized version, \eqref{eq:optim_reg}, can also be solved using the same gradient-based method.
The situation is especially simplified if the dataset is a set of $d$-dimensional vectors, i.e., $\mathcal{D}=\{\bm{x}_i\in\mathbb{R}^d \mid i=1,\dots,n\}$.
In the optimization process, sample $\hat{\mathcal{D}} = \{ \hat{\bm{x}}_i \in \mathbb{R}^d \mid i=1,\dots,m \}$ from $p_\theta(\bm{x})$ being learned, and use them for computing the the regularization term in \eqref{eq:reg} and their gradients.
Then, incorporate them into the original gradient-based updates.
These procedures are summarized in Algorithm~\ref{alg:main}.

The computation of $R_\mathcal{K}$ requires $O(m^2 \vert\mathcal{K}\vert)$ operations in a naive implementation, where $m$ is the number of samples drawn from $p_\theta(\bm{x})$.
This will be burdensome for a very large $\vert\mathcal{K}\vert$, which can be alleviated by carefully choosing $J^\text{ref}$'s so that the computed HSIC values can be reused for many times.

\begin{algorithm}[t]
    \renewcommand{\algorithmicrequire}{\textbf{Input:}}
    \renewcommand{\algorithmicensure}{\textbf{Output:}}
    \def\NoNumber#1{{\def\alglinenumber##1{}\State #1}\addtocounter{ALG@line}{-1}}
    \caption{Knowledge-regularized gradient method}
    \label{alg:main}
    {\begin{algorithmic}[1]
        \Require Data $\mathcal{D}=\{\bm{x}_i\}$, knowledge $\mathcal{K} = \{ (J_s^\text{ref},J_s^+,J_s^-)\}$, hyperparameters $\lambda \geq 0, \ \nu_\alpha \geq 0$, sample size $m$
        \Ensure A set of parameters $\theta^*$ of $p_\theta(\bm{x})$
        \State initialize $\theta$
        \Repeat
            \State draw $\hat{\mathcal{D}} = \{ \hat{\bm{x}}_i \in \mathbb{R}^d \mid i=1,\dots,m \}$ from $p_\theta(\bm{x})$ 
            %
            \For{$s=1,\dots,\vert \mathcal{K} \vert$}
                \State $J_s \leftarrow J^\text{ref}_s \cup J^+_s \cup J^-_s$
                \State $\hat{\mathcal{D}}_s \leftarrow \{ [\hat{\bm{x}}_i]_{J_s} \in \mathbb{R}^{\vert J_s \vert} \mid i=1,\dots,m \}$ 
                
                \Comment{$[\hat{\bm{x}}_i]_J$ is the subvector of $\hat{\bm{x}}_i$} indexed by $J$
                \State compute ${\partial \hat\rho_s}/{\partial \theta}$ using $\hat{\mathcal{D}}_s$ 
            \EndFor
            \State compute ${\partial R_\mathcal{K}}/{\partial \theta}$ using $\{ {\partial \hat\rho_s}/{\partial \theta} \mid s=1,\dots,\vert\mathcal{K}\vert \}$
            \State update $\theta$ using ${\partial L}/{\partial \theta} + \lambda {\partial R_\mathcal{K}}/{\partial \theta}$
        \Until{convergence}
    \end{algorithmic}}%
\end{algorithm}



\section{Related Work}
\label{related}

\subsection{Generative Modeling with Prior Knowledge}

A perspective related to this work is model design based on prior knowledge.
For example, the object-oriented Bayesian network language \citep{Koller97} helps graphical model designs when the structures behind data can be described in an object-oriented way.
Such tools are useful when we can prepare prior knowledge enough for building a full model.
However, it is not apparent how to utilize them when knowledge is only partially given, which is often the case in practice.
In contrast, our method can exploit prior knowledge even if only a part of the structures is known in advance.
Another related perspective is the structure learning of Bayesian networks with constraints from prior knowledge \citep{Cussens17,Li18}.

Posterior regularization (PR) \citep{Ganchev10,Zhu14,Mei14,Hu18} is known as a framework for incorporating prior knowledge into probabilistic models.
In fact, our method can be included in the most general class of PR, which was briefly mentioned in \citep{Zhu14}.
However, practically, existing work \citep{Ganchev10,Zhu14,Mei14,Hu18} only considered limited cases of PR, where constraints were written in terms of a linear monomial of expectation with regard to the target distribution.
In contrast, our method tries to fulfill the constraints on statistical dependence, which needs more complex expressions.

The work by Lopez \emph{et al.} [\citeyear{Lopez18}] should be understood as a kind of technical complement of ours (and \emph{vice versa}).
While Lopez \emph{et al.} [\citeyear{Lopez18}] regularize the amortized inference of VAEs to make \emph{independence of latent variables}, we regularize a general generative model itself to make \emph{dependence of observed variables} compatible with prior knowledge.
In other words, the former regularizes an encoder, whereas the latter regularizes a decoder.
Moreover, our method is more flexible than \cite{Lopez18} in terms of applicable type of knowledge; they considered only independence, but our method can incorporate both dependence and independence.
Also, while Lopez \emph{et al.} [\citeyear{Lopez18}] discussed the method only for VAEs, our method applies to any generative models as long as the mild assumptions in Section~\ref{target} are satisfied.

\subsection{Knowledge of Feature Dependence}
\label{useofknowledge}

In fact, the use of knowledge of feature relation has been discussed in different or more specific contexts.
In natural language processing, feature similarity is often available as the similarity between words.
Xie \emph{et al.} \citep{Xie15} exploited correlations of words for topic modeling, and Liu \emph{et al.} \citep{Liu15} incorporated semantic similarity of words in learning word embeddings.
These are somewhat related to generative modeling, but the scope of data and model is limited.
Moreover, there are several pieces of research on utilizing feature similarity \citep{Krupka07,Li08,Sandler09,Li17,Mollaysa17} for discriminative problems.
Despite these interests, there have been surprisingly few studies on the use of such knowledge for generative modeling in general.



\begin{figure}[t]
    \centering
    \hfill
    \begin{minipage}[t]{0.33\linewidth}
        \centering
        \vspace*{-2mm}
        \includegraphics[clip,height=2.1cm]{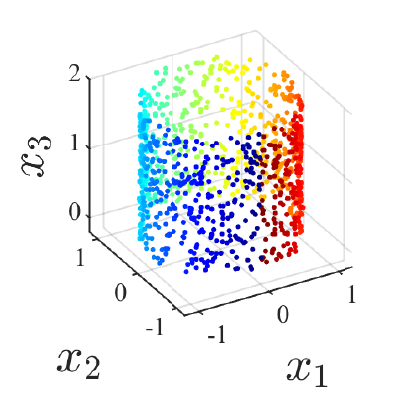}
        \vspace*{-1mm}
        \caption{\textsc{Toy} dataset.}
        \label{fig:unitoy}
    \end{minipage}
    \hfill
    \begin{minipage}[t]{0.6\linewidth}
        \centering
        \vspace*{0mm}
        \includegraphics[clip,height=1.8cm]{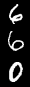}
        \includegraphics[clip,height=1.8cm]{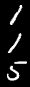}
        \includegraphics[clip,height=1.8cm]{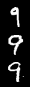}
        \includegraphics[clip,height=1.8cm]{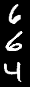}
        \includegraphics[clip,height=1.8cm]{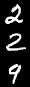}
        \includegraphics[clip,height=1.8cm]{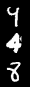}
        \caption{\textsc{ccMNIST} dataset.}
        \label{fig:ccmnist}
    \end{minipage}
    \hfill
\end{figure}

\section{Experiments}
\label{expt}


\subsection{Datasets and Prior Knowledge}
\label{expt:data}

We used the following four datasets, for which we can prepare plausible prior knowledge of feature relations.

\paragraph{Toy data.}
We created a cylinder-like toy dataset as exemplified in Figure~\ref{fig:unitoy}.
On this dataset, we know that $x_1$ and $x_2$ are more statistically dependent than $x_1$ and $x_3$ are.
This knowledge is expressed in the manner of \eqref{eq:def} as
\begingroup\makeatletter\def\f@size{9}\check@mathfonts
\begin{equation*}
    \mathcal{K}_\text{Toy} = \big\{ ( J^\text{ref}_1=\{1\}, \ J^+_1=\{2\}, \ J^-_1=\{3\} ) \big\}.
\end{equation*}
\endgroup

\paragraph{Constrained concatenated MNIST.}
We created a new dataset from MNIST, namely \emph{constrained concatenated MNIST} (\textsc{ccMNIST}, examples in Figure~\ref{fig:ccmnist}) as follows.
Each image in this dataset was generated by vertically concatenating three MNIST images.
Here, the concatenation is constrained so that the top and middle images have the same label, whereas the bottom image is independent of the top two.
We created training and validation sets from MNIST's training dataset and a test set from MNIST's test dataset.
On this dataset, we can anticipate that the top third and the middle third of each image are more dependent than the top and the bottom are.
If each image is vectorized in a row-wise fashion as a $2352$-dim vector, this knowledge is expressed as
\begingroup\makeatletter\def\f@size{9}\check@mathfonts
\begin{equation*}\begin{aligned}
    \mathcal{K}_\text{ccMNIST} = \big\{(
        &J^\text{ref}_1 = \{1,...,784\},
        \
        J^+_1 = \{785,...,1568\},
        \\
        &J^-_1 = \{1569,...,2352\}
    ) \big\}.
\end{aligned}\end{equation*}
\endgroup

\paragraph{Plant sensor data.}
We used a simulated sensor dataset (\textsc{Plant}), which corresponds to the example in Figure~\ref{fig:knowledge:partial}.
The simulation is based on a real industrial chemical plant called the Tennessee Eastman process \citep{Downs93}.
In the plant, there are four major unit operations: reactor, vapor-liquid separator, recycle compressor, and product stripper.
On each unit operation, sensors such as level sensors and thermometers are attached.
We used readings of 22 sensors.
On this dataset, we can anticipate the relative dependence between sets of sensors based on the process diagram of the Tennessee Eastman process \citep{Downs93}.
The knowledge set for this dataset ($\mathcal{K}_\text{Plant}$) contains, e.g.,
\begingroup\makeatletter\def\f@size{9}\check@mathfonts
\begin{equation*}\begin{aligned}
    &(J^\text{ref}_1=J_\text{compressor}, \ J^+_1=J_\text{reactor}, \ J^-_1=J_\text{stripper}), \dots 
\end{aligned}\end{equation*}
\endgroup
where $J_\text{compressor}$ is the set of indices of sensors attached to the compressor, and so on.
This is plausible because the reactor is directly connected to the compressor while the stripper is not.
We prepared $\mathcal{K}_\text{Plant}$ comprising $\vert\mathcal{K}_\text{Plant}\vert=12$ relations based on the structure of the plant \cite{Downs93}.

\paragraph{Solar energy production data.}
We used the records of solar power production\footnote{\url{www.nrel.gov/grid/solar-power-data.html}} (\textsc{Solar}) of 137 solar power plants in Alabama in June 2006.
The data of the first 20 days were used for training, the next five days were for validation, and the last five days were for test.
This dataset corresponds to the example in Figure~\ref{fig:knowledge:sideinfo} because we can anticipate the dependence of features from the pairwise distances between the solar plants.
We created a knowledge set $\mathcal{K}_\text{Solar}$ as follows;
for the $i$-th plant ($i=1,\dots,137$), if the nearest plant is within 10 [km] and the distance to the second-nearest one is more than 12 [km], then add $(\{i\}, \{j_i\}, \{k_i\})$ to $\mathcal{K}_\text{Solar}$, where $j_i$ and $k_i$ are the indices of the nearest and the second-nearest plants, respectively.
This resulted in $\vert\mathcal{K}_\text{Solar}\vert=31$.

\begin{table}[t]
    \centering
    \setlength{\tabcolsep}{6pt}
    {\footnotesize\begin{tabular}{ccccc}
        \toprule
        {$\dim(\bm{z})$} & {L2 only} & {out-layer} & {proposed} & {\scriptsize dedicated decoder} \\
        \midrule
        25 & $332$ ($3.4$) & $331$ ($3.6$) & $\mathbf{321}$ ($2.4$) & $\mathbf{325}$ ($1.5$) \\
        50 & $326$ ($3.3$) & $325$ ($3.1$) & $\mathbf{312}$ ($2.6$) & $\mathbf{301}$ ($4.3$) \\
        100 & $332$ ($2.9$) & $329$ ($2.1$) & $\mathbf{317}$ ($2.0$) & $\mathbf{305}$ ($5.9$) \\
        200 & $333$ ($4.4$) & $333$ ($3.2$) & $\mathbf{322}$ ($2.4$) & $\mathbf{305}$ ($4.2$) \\
        \bottomrule
    \end{tabular}}%
    \caption{Test mean cross-entropy of VAEs on \textsc{ccMNIST}. Only the case of $\dim(\text{MLP})=2^{10}$ is reported, and the other cases are similar.}
    \label{tab:vae:crossent}
\end{table}

\begin{table}[t]
    \vspace*{0mm}
    \centering
    \setlength{\tabcolsep}{7pt}
    {\footnotesize\begin{tabular}{ccccc}
        \toprule
        dataset & $\dim(\bm{z})$ & L2 only & out-layer & proposed \\
        \midrule
        \textsc{Toy} & 4 & $0.76$ ($.15$) & $\mathbf{0.41}$ ($.08$) & $\mathbf{0.41}$ ($.07$) \\ 
        \midrule
        \multirow{3}*{\textsc{Plant}}
        & 4 & $8.73$ ($.17$) & --- & $8.67$ ($.10$) \\
        & 7 & $8.21$ ($.15$) & --- & $\mathit{8.06}$ ($.16$) \\
        & 11 & $8.11$ ($.23$) & --- & $\mathit{7.95}$ ($.17$) \\
        \midrule
        \multirow{3}*{\textsc{Solar}}
        & 4 & $4.51$ ($.96$) & $4.32$ ($.46$) & $\mathbf{2.57}$ ($.33$) \\
        & 13 & $3.22$ ($.54$) & $3.15$ ($.47$) & $\mathbf{1.97}$ ($.16$) \\
        & 54 &  $2.73$ ($.43$) & $2.71$ ($.49$) & $\mathbf{1.90}$ ($.16$) \\
        \bottomrule
        \multicolumn{5}{r}{{\scriptsize Significant diff. from ``L2 only'' with  \ \textit{italic}: $p<.01$ \ or \ \textbf{bold}: $p<.001$.}}
    \end{tabular}}%
    \caption{Test mean per-feature reconstruction errors by VAEs on the three datasets. Here only the cases of $\dim(\text{MLP})=2^5,2^6,2^7$ (resp. for the three datasets) are reported. Other cases are similar.}
    \label{tab:vae:recerr}
\end{table}

\subsection{Settings}

\paragraph{Models.}
We used factor analysis (FA, see also \citet{Zhang09}), VAE, and GAN.
We trained FA and VAEs on all datasets and GANs on the \textsc{Toy} dataset.
For VAEs and GANs, encoders, decoders, and discriminators were modeled using multi-layer perceptrons (MLPs) with three hidden layers.
Hereafter we denote the dimensionality of VAE/GAN's latent variable by $\dim(\bm{z})$ and the dimensionality of MLP's hidden layer by $\dim(\text{MLP})$.
We tried $\dim(\text{MLP})=2^5,2^6,\dots,2^{10}$ and used several values of $\dim(\bm{z})$ for the different datasets.

\paragraph{Baselines.}
As the simplest baseline, we trained every model only with L2 regularization (i.e., weight decay).
We tried another baseline, in which the weight of the last layer of decoder MLPs are regularized using $\mathcal{K}$; e.g., for $\mathcal{K}_\text{Toy} = \{ \{1\}, \{2\}, \{3\} \}$, we used a regularization term
\begingroup\makeatletter\def\f@size{9}\check@mathfonts
\begin{equation*}
    R^\text{out-layer}_{\mathcal{K}_\text{toy}} := \max(0, \ \epsilon + \Vert \bm{w}_1-\bm{w}_2 \Vert_2^2 - \Vert \bm{w}_1-\bm{w}_3 \Vert_2^2),
\end{equation*}
\endgroup
where $\bm{w}_i$ denotes the $i$-th row of the weight matrix of the last layer of the decoder MLP, and $\epsilon$ was tuned in the same manner as $\nu_\alpha$.
We refer to this baseline as ``out-layer.''
It cannot be applied to the \textsc{Plant} dataset because the feature sets in $\mathcal{K}_\text{Plant}$ do not have any one-to-one correspondences.

\paragraph{Hyperparameters.}
We computed HSIC with $m=128$ using Gaussian kernels with the bandwidth set by the median heuristics.
No improvement was observed with $m>128$.
The hyperparameters were chosen based on the performance on the validation sets.
The search was not intensive; $\lambda$ was chosen from three candidate values that roughly adjust orders of $L$ and $R_\mathcal{K}$, and $\nu_\alpha$ was chosen from .01 or .05.

\subsection{Results}

Below, the quantitative results are reported mainly with cross entropy (for \textsc{ccMNIST}) or reconstruction errors (for the others) as they are a universal performance criterion to examine the generalization capability.
For VAEs, the significance of the improvement by the proposed method did not change even when we examined the ELBO values.

\paragraph{Evaluation by test set performance.}
The test set performance of VAEs are shown in Tables~\ref{tab:vae:crossent} and \ref{tab:vae:recerr}.
We can observe that, while the improvement by the out-layer baseline was quite marginal, the proposed regularization method resulted in significant improvement.
The performance of FA slightly improved (details omitted due to space limitations).

\paragraph{Comparison to a dedicated model.}
We compared the performance of the proposed method to that of a model designed specifically for \textsc{ccMNIST} dataset (termed ``dedicated'').
The dedicated model uses an MLP designed in accordance with the prior knowledge that the top and the middle parts are from the same digit.
The performance of the dedicated model is shown in the right-most column of Table~\ref{tab:vae:crossent}.
We can observe that the proposed method performed intermediately between the most general case (L2 only) and the dedicated model.
As designing models meticulously is time-consuming and often even infeasible, the proposed method will be useful to give a trade-off between performance and a user's workload.


\begin{figure}[t]
    \begin{minipage}[t]{0.48\linewidth}
        \centering
        \vspace*{0pt}
        \begin{tikzpicture}[]
            \tikzstyle{every node}=[font=\scriptsize]
            \begin{axis}[
                width=4cm, height=3.2cm,
                xlabel={$\vert \mathcal{K}_\text{Solar} \vert$},
                ylabel={Test rec. errors},
                xlabel near ticks,
                ylabel near ticks,
                ymajorgrids=true, xmajorgrids=true]
                \addplot+[error bars/.cd, y dir=both, y explicit] table[y error index=2] {vae2_solar_ratios.txt};
            \end{axis}
        \end{tikzpicture}
        \caption{Performance of VAEs learned on the \textsc{Solar} dataset with $\mathcal{K}_\text{Solar}$ of different sizes.}
        \label{fig:ratios}
    \end{minipage}
    \hfill
    \begin{minipage}[t]{0.45\linewidth}
        \centering
        \vspace*{-0.5mm}
        \includegraphics[clip,height=1.38cm]{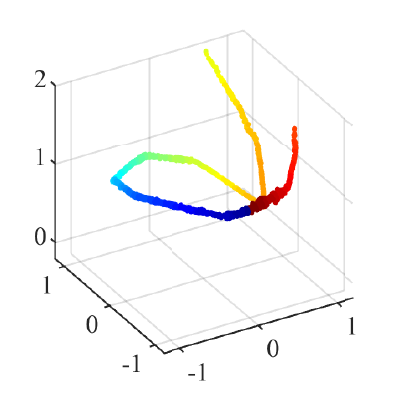}
        \hspace{1mm}
        \includegraphics[clip,height=1.38cm]{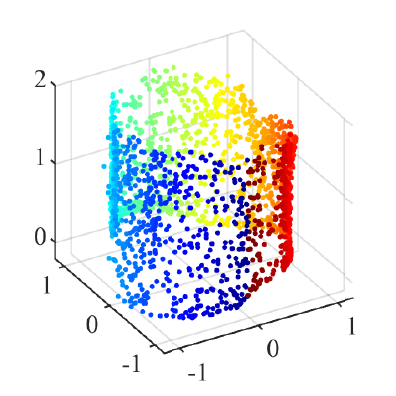}
        \\
        \vspace*{-1mm}
        \includegraphics[clip,height=1.38cm]{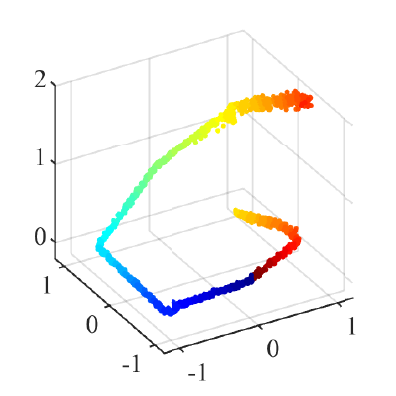}
        \hspace{1mm}
        \includegraphics[clip,height=1.38cm]{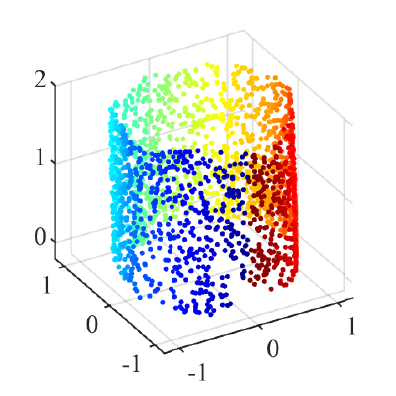}
        \vspace*{-1mm}
        \caption{Samples from GAN: (\emph{left}) w/o and (\emph{right}) with the proposed method.}
        \label{fig:toy}
    \end{minipage}
\end{figure}

\paragraph{Effect of knowledge set size.}
Another interest lies in how the amount of provided prior knowledge affects the performance of the regularizer.
We investigated this by changing the number of tuples in $\mathcal{K}_\text{Solar}$ used in learning VAEs.
We prepared knowledge subsets by extracting some tuples from the original set $\mathcal{K}_\text{Solar}$.
When creating the subsets, tuples were chosen so that a larger subset contained all elements of the smaller ones.
Figure~\ref{fig:ratios} shows the test set performances along with the subset size.
The performance is improved with a larger knowledge set.
For example, the difference between $\vert\mathcal{K}\vert=6$ and $\vert\mathcal{K}\vert=31$ is significant with $p\approx.003$.

\paragraph{Inspection of generated samples.}
We inspected the samples drawn from GANs trained on the \textsc{Toy} dataset.
When the proposed regularizer was not applied, we often observed a phenomenon similar to ``mode collapse'' (see, e.g., \cite{Metz17}) occurred as in the left-side plots of Figure~\ref{fig:toy}.
In contrast, the whole geometry of the \textsc{Toy} dataset was always captured successfully when the proposed regularizer was applied, as in the right-side plots of Figure~\ref{fig:toy}.

\paragraph{Running time.}
For example, the time needed for training a VAE for 50 epochs on the \textsc{ccMNIST} dataset was 150 or 117 seconds, with or without the proposed method, respectively.
As the complexity is linear in $\vert\mathcal{K}\vert$, this would not be prohibitive, while any speed-up techniques will be useful.


\section{Conclusion}
\label{concl}

In this work, we developed a method for regularizing generative models using prior knowledge of feature dependence, which is frequently encountered in practice and has been studied in contexts other than generative modeling.
The proposed regularizer can be incorporated in off-the-shelf learning methods of many generative models.
The direct extension of the current method includes the use of higher-order dependence between multiple sets of features \citep{Pfister17}.


\section*{Acknowledgments}
This work was supported by JSPS KAKENHI Grant Numbers JP19K21550, JP18H06487, and JP18H03287.

{\fontsize{10pt}{10.1pt}\selectfont\bibliography{main}
\bibliographystyle{named}}


\appendix

\section{Detailed settings of experiments}

\subsection{Datasets and knowledge sets}

\subsubsection{Toy data}

In creating the \textsc{Toy} data, we generated 1,000 samples for a training set, 2,000 samples for a validation set, and another 2,000 samples for a test set.

\subsubsection{Concatenated MNIST}

In creating the \textsc{ccMNIST} data, we generated a training set of size 20,000 and a validation set of size 10,000 from MNIST's training dataset and a test set of size 10,000 from MNIST's test dataset.

\subsubsection{Plant sensor data}

\paragraph{Dataset}
We used the Tennessee Eastman (TE) process data available online\footnote{\url{github.com/camaramm/tennessee-eastman-profBraatz} [retrieved 30 December 2018].}, which were generated using \texttt{teprob.f}\footnote{The code and the related materials are also archived online at \url{depts.washington.edu/control/LARRY/TE/download.html}.} originally provided by the authors of \citep{Downs93}.
We downloaded \texttt{d00.dat} and \texttt{d00\_te.dat}, which contained the data generated with the normal operating condition, and created three sets from them.
We used the whole \texttt{d00.dat} for a training set (size 500) and divided \texttt{d00\_te.dat} into a validation set (size 400) and a test set (size 560).
The original data consist of 12 manipulated variables and 41 process measurement variables.
Within the process measurement variables, we used the measurements of the level sensors, flow rate sensors, thermometers, and manometers, which finally composed the 22-dimensional data.
As preprocessing, we normalized the data using the values of average and standard deviation of each variable of the training set. Moreover, we added noise following $\mathcal{N}(0,10^{-4})$.

\paragraph{Knowledge set}
Based on the sensor locations shown in the process diagram of the TE process \citep{Downs93}, we classified each of the 22 variables into one of the following seven groups: Feed~1, Feed~2, Reactor, Vapor-liquid separator, Recycle compressor, Product stripper, and Purge.
This grouping is built upon the unit structure of the TE process.
The units are interconnected in the process as shown in Figure~\ref{fig:tep}, according to which we designed the knowledge set, $\mathcal{K}_\text{Plant}$, as shown in the below of Figure~\ref{fig:tep}.

\begin{figure*}[p]
    \centering
    \begin{minipage}{\textwidth}
        \centering
        \includegraphics[width=14cm,clip]{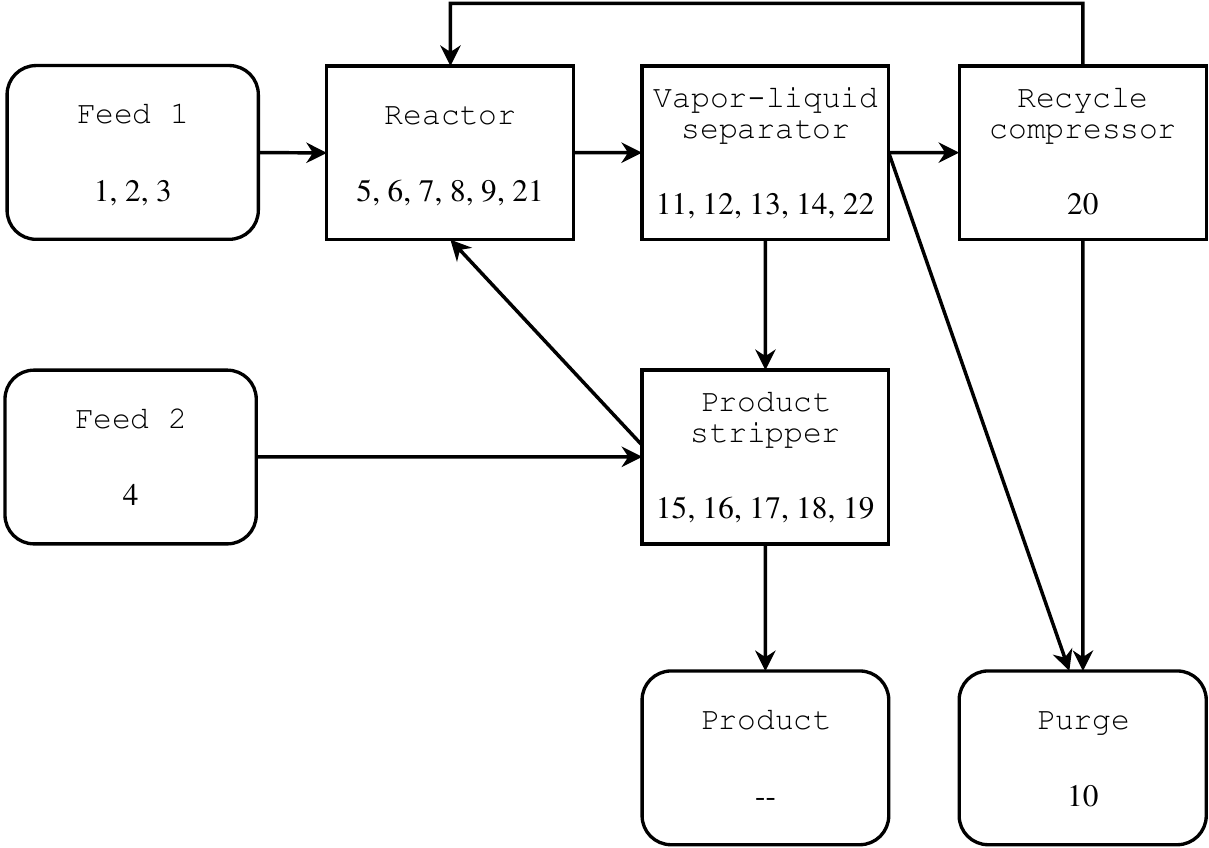}
        \caption{Diagram of the units in the TE process drawn by the authors referring to the original one. The numbers below the unit names are the sensor indices corresponding to each unit. The grouping of the sensors is by the authors of this manuscript and not necessarily complete nor general. Correspondingly, the following knowledge set was used in the experiments.}
        \label{fig:tep}
    \end{minipage}
    \par
    \vspace*{2ex}
    \begin{minipage}{\textwidth}
        \begin{equation*}\begin{aligned}
            \mathcal{K}_\text{Solar} = \big\{&~
            (J_\text{Reactor}, \ J_\text{Separator}, \ J_\text{Feed2}), ~~
            (J_\text{Reactor}, \ J_\text{Separator}, \ J_\text{Purge}), ~~
            (J_\text{Separator}, \ J_\text{Stripper}, \ J_\text{Feed1}),
            \\ &~
            (J_\text{Separator}, \ J_\text{Stripper}, \ J_\text{Feed2}), ~~
            (J_\text{Separator}, \ J_\text{Compressor}, \ J_\text{Feed1}), ~~
            (J_\text{Separator}, \ J_\text{Compressor}, \ J_\text{Feed2}),
            \\ &~
            (J_\text{Compressor}, \ J_\text{Reactor}, \ J_\text{Feed1}), ~~
            (J_\text{Compressor}, \ J_\text{Reactor}, \ J_\text{Feed2}), ~~
            (J_\text{Compressor}, \ J_\text{Reactor}, \ J_\text{Stripper}),
            \\ &~
            (J_\text{Stripper}, \ J_\text{Reactor}, \ J_\text{Feed1}), ~~
            (J_\text{Stripper}, \ J_\text{Reactor}, \ J_\text{Purge}), ~~
            (J_\text{Stripper}, \ J_\text{Reactor}, \ J_\text{Compressor})
            ~ \big\},
        \end{aligned}\end{equation*}
        where
        \begin{gather*}
            J_\text{Feed1} = \{1,2,3\}, \quad
            J_\text{Feed2} = \{4\}, \quad
            J_\text{Reactor} = \{5,6,7,8,9,21\}, \quad
            J_\text{Separator} = \{11,12,13,14,22\}, \\
            J_\text{Stripper} = \{15,16,17,18,19\}, \quad
            J_\text{Compressor} = \{20\}, \quad\text{and}\quad
            J_\text{Purge} = \{10\}.
        \end{gather*}
    \end{minipage}
\end{figure*}

\subsubsection{Solar}

\paragraph{Dataset}
The \textsc{Solar} dataset was created from the data provided at the website of National Renewable Energy Laboratory\footnote{\url{www.nrel.gov/grid/solar-power-data.html} [retrieved 9 January 2019].}.
We downloaded the set of csv files (\texttt{al-pv-2006.zip}) containing power production records of solar plants in Alabama state and concatenated the contents of the 137 files as the 137-dimensional dataset.
We used the data of June 2006 because the seasonal variation seemed moderate in that period.
We subsampled the original data (sampled every five minutes) by 1/3 and divided them into a training set (days 1--20), a validation set (days 21--25), and a test set (days 26--30).
Finally the size of each set is 1,920, 480, and 480, respectively.
Because the data contain many zero values, we added $10^{-3}$ to the whole datasets and took natural logarithm, which resulted in the data valued approximately from $-6$ to $4$.
As preprocessing, we normalized the data using the values of average and standard deviation of each variable of the training set. Moreover, we added noise following $\mathcal{N}(0,10^{-4})$.

\paragraph{Knowledge set}
Knowledge set was created following scheme presented in the main manuscript.


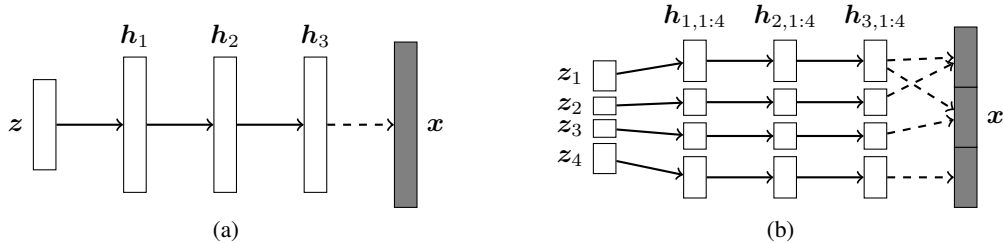
\begin{figure*}[t]
    \centering
    \subfloat[]{
        \begin{tikzpicture}
            \node (rect) [fill=white,minimum width=.3cm,minimum height=1.2cm,draw,label=left:{$\bm{z}$}] (z) at (0,0) {};
            \node (rect) [fill=white,minimum width=.3cm,minimum height=1.8cm,draw,label=above:{$\bm{h}_1$}] (h1) at (1.2,0) {};
            \node (rect) [fill=white,minimum width=.3cm,minimum height=1.8cm,draw,label=above:{$\bm{h}_2$}] (h2) at (2.4,0) {};
            \node (rect) [fill=white,minimum width=.3cm,minimum height=1.8cm,draw,label=above:{$\bm{h}_3$}] (h3) at (3.6,0) {};
            \node (rect) [fill=gray,minimum width=.3cm,minimum height=2.2cm,draw,label=right:{$\bm{x}$}] (x) at (4.8,0) {};
            \draw[thick,->] (z) -- (h1);
            \draw[thick,->] (h1) -- (h2);
            \draw[thick,->] (h2) -- (h3);
            \draw[thick,->,dashed] (h3) -- (x);
        \end{tikzpicture}
        \label{fig:spdec:gen}
    }
    \hspace{1cm}
    \subfloat[]{
        \begin{tikzpicture}
            \node (rect) [fill=white,minimum width=.3cm,minimum height=0.4cm,draw,label=left:{$\bm{z}_1$}] (z1) at (0,0.5) {};
            \node (rect) [fill=white,minimum width=.3cm,minimum height=0.2cm,draw,label=left:{$\bm{z}_2$}] (z2) at (0,0.1) {};
            \node (rect) [fill=white,minimum width=.3cm,minimum height=0.2cm,draw,label=left:{$\bm{z}_3$}] (z3) at (0,-0.2) {};
            \node (rect) [fill=white,minimum width=.3cm,minimum height=0.4cm,draw,label=left:{$\bm{z}_4$}] (z4) at (0,-0.6) {};
            \node (rect) [fill=white,minimum width=.3cm,minimum height=0.55cm,draw,label=above:{$\bm{h}_{1,1:4}$}] (h11) at (1.2,0.7) {};
            \node (rect) [fill=white,minimum width=.3cm,minimum height=0.35cm,draw] (h12) at (1.2,0.15) {};
            \node (rect) [fill=white,minimum width=.3cm,minimum height=0.35cm,draw] (h13) at (1.2,-0.3) {};
            \node (rect) [fill=white,minimum width=.3cm,minimum height=0.55cm,draw] (h14) at (1.2,-0.85) {};
            \node (rect) [fill=white,minimum width=.3cm,minimum height=0.55cm,draw,label=above:{$\bm{h}_{2,1:4}$}] (h21) at (2.4,0.7) {};
            \node (rect) [fill=white,minimum width=.3cm,minimum height=0.35cm,draw] (h22) at (2.4,0.15) {};
            \node (rect) [fill=white,minimum width=.3cm,minimum height=0.35cm,draw] (h23) at (2.4,-0.3) {};
            \node (rect) [fill=white,minimum width=.3cm,minimum height=0.55cm,draw] (h24) at (2.4,-0.85) {};
            \node (rect) [fill=white,minimum width=.3cm,minimum height=0.55cm,draw,label=above:{$\bm{h}_{3,1:4}$}] (h31) at (3.6,0.7) {};
            \node (rect) [fill=white,minimum width=.3cm,minimum height=0.35cm,draw] (h32) at (3.6,0.15) {};
            \node (rect) [fill=white,minimum width=.3cm,minimum height=0.35cm,draw] (h33) at (3.6,-0.3) {};
            \node (rect) [fill=white,minimum width=.3cm,minimum height=0.55cm,draw] (h34) at (3.6,-0.85) {};
            \node (rect) [fill=gray,minimum width=.3cm,minimum height=0.8cm,draw] (x1) at (4.8,-0.85) {};
            \node (rect) [fill=gray,minimum width=.3cm,minimum height=0.8cm,draw,label=right:{$\bm{x}$}] (x2) at (4.8,-0.05) {};
            \node (rect) [fill=gray,minimum width=.3cm,minimum height=0.8cm,draw] (x3) at (4.8,0.75) {};
            \draw[thick,->] (z1) -- (h11);
            \draw[thick,->] (h11) -- (h21);
            \draw[thick,->] (h21) -- (h31);
            \draw[thick,->] (z2) -- (h12);
            \draw[thick,->] (h12) -- (h22);
            \draw[thick,->] (h22) -- (h32);
            \draw[thick,->] (z3) -- (h13);
            \draw[thick,->] (h13) -- (h23);
            \draw[thick,->] (h23) -- (h33);
            \draw[thick,->] (z4) -- (h14);
            \draw[thick,->] (h14) -- (h24);
            \draw[thick,->] (h24) -- (h34);
            \draw[thick,->,dashed] (h31) -- (x3);
            \draw[thick,->,dashed] (h32) -- (x3);
            \draw[thick,->,dashed] (h31) -- (x2);
            \draw[thick,->,dashed] (h33) -- (x2);
            \draw[thick,->,dashed] (h34) -- (x1);
        \end{tikzpicture}
        \label{fig:spdec:sp}
    }
    \caption{Architectures of (a) the decoder modeled by a general fully-connected MLP and (b) the decoder specially designed for \textsc{ccMNIST} dataset. In both diagrams, $\bm{z}$ is the latent variable of VAE, $\bm{h}$'s are the hidden layers of the decoder, and $\bm{x}$ denotes the observed variable. The solid arrows represent affine transformation followed by the ReLU function, and the dashed arrows represent affine transformation.}
    \label{fig:spdec}
\end{figure*}

\subsection{Generative models}

\subsubsection{Factor analysis model}

In the factor analysis (FA) model we used, the observed variable, $\bm{x}\in\mathbb{R}^d$, is modeled as
\begin{equation}
    \bm{x} \approx \bm{W} \bm{z} + \bm\mu,
\end{equation}
where $\bm{z}\in\mathbb{R}^p$ ($p \leq d$) is a latent variable, which is inferred by
\begin{equation}
    \tilde{\bm{z}} = \bm{W}^\mathsf{T} (\bm{x}-\bm\mu).
\end{equation}
The parameters, $\bm{W}$ and $\bm\mu$, are learned by minimizing the following square loss:
\begin{equation}
    \Vert \bm{x} - \bm{W} \bm{W}^\mathsf{T} (\bm{x}-\bm\mu) - \bm\mu \Vert_2^2.
\end{equation}
This model can be regarded as a linear autoencoder with the constrained encoder and encoder.
Other detailed settings are described in Section~\ref{exptsetting}.

\subsubsection{Variational autoencoder}

A variational autoencoder (VAE) \citep{Kingma14} is learned via maximization of the evidence lower bound (ELBO):
\begin{multline}
    \text{ELBO} =
    \mathbb{E}_{q_{\phi_\text{enc}}(\bm{z} \mid \bm{x})} [ \log p_{\theta_\text{dec}}(\bm{x} \mid \bm{z}) ]
    \\
    - \mathrm{KL}( q_{\phi_\text{enc}}(\bm{z} \mid \bm{x}) \ \Vert \ p_{\theta_\text{dec}}(\bm{x} \mid \bm{z}) ),
\end{multline}
where $p_{\theta_\text{dec}}(\bm{x} \mid \bm{z})$ and $q_{\phi_\text{enc}}(\bm{z} \mid \bm{x})$ are a decoder and an encoder having sets of parameters, $\theta_\text{dec}$ and $\phi_\text{enc}$, respectively.
We modeled the encoder and the decoder using multi-layer perceptrons (MLPs) in every experiment.
Other detailed settings are described in Section~\ref{exptsetting}.

\subsubsection{Generative adversarial network}

A generative adversarial network (GAN) \citep{Goodfellow14} is learned via minimization of the following cross-entropy loss:
\begin{equation}
    \log D_{\phi_D} (\bm{x}) + \log \big( 1 - D_{\phi_D} \big( G_{\theta_G} (\bm{z}) \big) \big),
\end{equation}
where $G_{\theta_G}$ and $D_{\phi_D}$ are a generator and a discriminator with having sets of parameters, $\theta_G$ and $\phi_D$, respectively.
In GANs, the generator generates fake samples while the discriminator tries to distinguish the fake and the real data.
We modeled the generator and the discriminator using MLPs.
Other detailed settings are described in Section~\ref{exptsetting}.


\subsection{Experimental settings}
\label{exptsetting}

For settings that are not described below, we used the default values of PyTorch 0.4.1.

\subsubsection{Hyperparameter tuning}

The hyperparameters of the proposed method were chosen in accordance with the validation set performance.
The set of candidates for $\nu_\alpha$ was just $\{.01, \ .05\}$ for every experiment.
The sets of candidates for $\lambda$ are shown in Table~\ref{tab:cand}. The candidate values of $\lambda$ were set so that the orders of the original loss function and the regularizer were roughly adjusted.

\subsubsection{Architectures}

The encoder / decoder of VAEs and the generator / discriminator of GANs were modeled using MLPs.
Each MLP has fully-connected three hidden layers, each of which has the same number of units.
The dimensionality of $\bm{z}$ and the number of units of MLP's hidden layers are summarized in Table~\ref{tab:dim}.
The dimensionality of $\bm{z}$ was determined basically in accordance with the cumulative contributing rates (i.e., explained variance) of PCA on the \textsc{Plant} and \textsc{Solar} datasets.
As for the \textsc{Plant} dataset, 4 principal components (PCs) explain 99\% variance, 7 PCs explain 99.9\%, and 11 PCs explain 99.99\%; analogously for the \textsc{Solar} dataset.

For the MLPs, we used ReLU as the activation function except for the VAE on the \textsc{Solar} dataset.
For the VAE on the \textsc{Solar} dataset, we used $\operatorname{tanh}$.

In the experiment of VAE on the \textsc{ccMNIST} dataset, for a reference of performance, we also tried a dedicated decoder architecture.
In the dedicated decoder, the latent variable, $\bm{z}$, was divided into four parts as $\bm{z}=[\bm{z}_1^\mathsf{T} ~ \bm{z}_2^\mathsf{T} ~ \bm{z}_3^\mathsf{T} ~ \bm{z}_4^\mathsf{T}]^\mathsf{T}$.
Then, the top third of each image is modeled only with $\bm{z}_1$ and $\bm{z}_2$, the middle third is with $\bm{z}_1$ and $\bm{z}_3$, and the bottom third is with $\bm{z}_4$.
The detail of the dedicated decoder is shown in Figure~\ref{fig:spdec:sp}.
This dedicated structure reflects the fact that the top and the middle parts of the images are from the same digit while the bottom is independent from the top two.

\subsubsection{Optimization}

Every optimization was done using Adam optimizer.
$\alpha$, one of the parameters of Adam, was basically 0.001 with a few exception; we used $\alpha=0.0005$ for learning the FA model on the \textsc{Plant} dataset and for the GAN on the \textsc{Toy} dataset.
Also, we used $\beta_1=0.5$ for GANs and $\beta_1=0.9$ otherwise. 


\section{Detailed experimental results}

\subsection{Computational speed}

We examined the computational time of training VAEs (using a Tesla V100 and the implementation based on PyTorch) with and without the proposed regularization.
We show the runtime of training VAE on the \textsc{ccMNIST} dataset in Table~\ref{tab:speed1} and the \textsc{Solar} dataset in Table~\ref{tab:speed2}.
The settings are respectively from Sections 5.3.1 and 5.3.3.
In both tables, the averages over 10 trials are shown.
In summary, training with the proposed regularizer is not prohibitively slow compared to the case without the regularizer ($q=0$ in Table~\ref{tab:speed1} and $\vert \mathcal{K} \vert=0$ in Table~\ref{tab:speed2}).
Also, the runtime is not so sensitive to $q>0$ and admissible for a middle-sized knowledge set.

\subsection{Evaluation by test set performance}

In Tables~\ref{tab:result1}, \ref{tab:result2}, and \ref{tab:result3}, we present the full version of the experimental results introduced in the main manuscript.

\begin{table*}[p]
    \centering
    \begin{minipage}{\textwidth}
        \centering
        \caption{Sets of candidates for hyperparameter $\lambda$, from which $\lambda$ was chosen by validation performance.}\label{tab:cand}
        \begin{tabular}{ccc}
            \toprule
            Model & Dataset & Candidates for $\lambda$ \\
            \midrule
            \midrule
            \multirow{3}*{FA} & \textsc{ccMNIST} & $5\times10^3$, $10^4$, or $2\times10^4$ \\
            & \textsc{Plant} & $2.5\times10^4$, $5\times10^4$, or $10^5$ \\
            & \textsc{Solar} & $5\times10^4$, $10^5$, or $2\times10^5$ \\
            \midrule
            \multirow{4}*{VAE} & \textsc{Toy} & $250$, $500$, or $1000$  \\
            & \textsc{ccMNIST} & $5\times10^3$, $10^4$, or $2\times10^4$ \\
            & \textsc{Plant} & $2.5\times10^4$, $5\times10^4$, or $10^5$ \\
            & \textsc{Solar} & $2.5\times10^5$, $5\times10^5$, or $10^6$ \\
            \bottomrule
        \end{tabular}
    \end{minipage}
    \par
    \vspace*{8ex}
    \begin{minipage}{\textwidth}
        \centering
        \caption{Settings of dimensionality of $\bm{z}$ and MLP's hidden layers used in the experiments.}\label{tab:dim}
        \begin{tabular}{cccc}
            \toprule
            Model & Dataset & $\dim(\bm{z})$ & $\dim(\text{MLP})$ \\
            \midrule
            \midrule
            \multirow{3}*{FA} & \textsc{ccMNIST} & $200$, $800$ & --- \\
            & \textsc{Plant} & $7$, $11$ & --- \\
            & \textsc{Solar} & $13$, $54$ & --- \\
            \midrule
            \multirow{4}*{VAE} & \textsc{Toy} & $4$ & $32$ \\
            & \textsc{ccMNIST} & $25$, $50$, $100$, $200$ & $512$, $1024$, $2048$ \\
            & \textsc{Plant} & $4$, $7$, $11$ & $32$, $64$, $128$ \\
            & \textsc{Solar} & $4$, $13$, $54$ & $128$, $256$, $512$ \\
            \midrule
            GAN & \textsc{Toy} & 4 & 32 \\
            \bottomrule
        \end{tabular}
    \end{minipage}
    \par
    \vspace*{8ex}
    \begin{minipage}{0.58\textwidth}
        \centering
        \vspace*{0cm}
        \caption{Runtime of training VAE on the \textsc{ccMNIST} for 50 epochs and the test performance.}
        \label{tab:speed1}
        \begin{tabular}{ccc}
            \toprule
            $q$ & Runtime [sec] & Test negative ELBO \\
            \midrule
            \midrule
            0 (without reg.) & 117 & $3.71 \times 10^2$ \\
            64 & 148 & $3.64 \times 10^2$ \\
            128 & 150 & $3.63 \times 10^2$ \\
            256 & 154 & $3.62 \times 10^2$ \\
            512 & 161 & $3.62 \times 10^2$ \\
            \bottomrule
        \end{tabular}
    \end{minipage}
    \hfill
    \begin{minipage}{0.38\textwidth}
        \centering
        \vspace*{0cm}
        \caption{Runtime of training VAE on the \textsc{Solar} for 100 epochs.}
        \label{tab:speed2}
        \begin{tabular}{cc}
            \toprule
            $\vert\mathcal{K}\vert$ & Runtime [sec] \\
            \midrule
            \midrule
            0 (without reg.) & 13.0 \\
            6 & 21.2 \\
            12 & 29.0 \\
            19 & 38.0 \\
            25 & 46.1 \\
            31 & 53.8 \\
            \bottomrule
        \end{tabular}
    \end{minipage}
\end{table*}

\begin{table*}[p]
    \centering
    \caption{Test mean per-feature reconstruction errors by FA models. Averages (standard deviations) over 10 random trials are shown.}
    \label{tab:result1}
    \setlength{\tabcolsep}{8pt}
    {\small\begin{tabular}{ccccc}
        \toprule
        \multicolumn{2}{c}{Setting} & \multicolumn{3}{c}{Results} \\
        \cmidrule(lr){1-2}\cmidrule(lr){3-5}
        Dataset & $\dim(\bm{z})$ & L2 only & L2 + out-layer & \textbf{L2 + proposed} \\
        \midrule\midrule
        \multirow{2}*{\textsc{ccMNIST}}
        & {$200$} & $1.970 \times 10^{-5}$ ($3 \times 10^{-9}$) & $1.970 \times 10^{-5}$ ($3 \times 10^{-9}$) & $1.958 \times 10^{-5}$ ($1 \times 10^{-8}$) ${}^{***}$ \\
        & {$800$} & $8.150 \times 10^{-6}$ ($2 \times 10^{-9}$) & $8.148 \times 10^{-6}$ ($3 \times 10^{-9}$) & $7.828 \times 10^{-6}$ ($2 \times 10^{-8}$) ${}^{***}$ \\
        \midrule
        \multirow{2}*{\textsc{Plant}}
        & $7$ & $6.318$ ($0.021$) & --- & $6.323$ ($0.019$) ${}^{~~~~~}$ \\
        & {$11$} & $4.949$ ($0.039$) & --- & $4.936$ ($0.025$) ${}^{~~~~~}$ \\
        \midrule
        \multirow{2}*{\textsc{Solar}}
        & {$13$} & $1.232$ ($0.006$) & $1.232$ ($0.006$) & $1.227$ ($0.005$) ${}^{*~~~}$ \\
        & {$54$} & $0.842$ ($0.004$) & $0.841$ ($0.004$) & $0.831$ ($0.004$) ${}^{***}$ \\
        \bottomrule
    \end{tabular}}%
    \par
    \vspace*{4ex}
    \caption{Test cross-entropy by VAEs on the \textsc{ccMNIST} dataset.}
    \label{tab:result2}
    \setlength{\tabcolsep}{8pt}
    {\small\begin{tabular}{ccccccc}
        \toprule
        \multicolumn{3}{c}{Setting} & \multicolumn{4}{c}{Results} \\
        \cmidrule(lr){1-3}\cmidrule(lr){4-7}
        Dataset & $\dim(\bm{z})$ & $\dim(\text{MLP})$ & L2 only & L2 + out-layer & \textbf{L2 + proposed} & [dedicated decoder] \\
        \midrule\midrule
        \multirow{9}*{\textsc{ccMNIST}}
        & 25 & 512 & $339.1$ ($3.6$) & $339.1$ ($2.7$) ${}^{~~~}$ & $325.1$ ($2.0$) ${}^{***}$ & [$335.5$ ($1.7$)] \\
        & 25 & 1024 & $332.0$ ($3.4$) & $330.5$ ($3.6$) ${}^{~~~}$ & $320.7$ ($2.4$) ${}^{***}$ & [$325.2$ ($1.5$)] \\
        & 25 & 2048 & $328.6$ ($2.6$) & $328.5$ ($2.8$) ${}^{~~~}$ & $321.5$ ($1.5$) ${}^{***}$ & [$324.6$ ($1.6$)] \\
        & 50 & 512 & $339.0$ ($4.8$) & $338.2$ ($3.9$) ${}^{~~~}$ & $321.8$ ($2.4$) ${}^{***}$ & [$326.1$ ($3.2$)] \\
        & 50 & 1024 & $325.8$ ($3.3$) & $325.3$ ($3.1$) ${}^{~~~}$ & $312.2$ ($2.6$) ${}^{***}$ & [$300.9$ ($4.3$)] \\
        & 50 & 2048 & $322.6$ ($3.5$) & $321.7$ ($3.1$) ${}^{~~~}$ & $313.8$ ($2.6$) ${}^{***}$ & [$290.2$ ($1.7$)] \\
        & 100 & 512 & $339.6$ ($1.5$) & $339.6$ ($1.0$) ${}^{~~~}$ & $324.8$ ($3.4$) ${}^{***}$ & [$326.8$ ($4.1$)] \\
        & 100 & 1024 & $331.6$ ($2.9$) & $329.4$ ($2.1$) ${}^{*~~}$ & $317.2$ ($2.0$) ${}^{***}$ & [$304.9$ ($5.9$)] \\
        & 100 & 2048 & $323.5$ ($3.6$) & $323.9$ ($2.8$) ${}^{~~~}$ & $317.4$ ($2.8$) ${}^{**~}$ & [$288.9$ ($2.2$)] \\
        & 200 & 512 & $343.0$ ($2.6$) & $340.4$ ($1.5$) ${}^{**}$ & $329.9$ ($2.2$) ${}^{***}$ & [$328.3$ ($5.4$)] \\
        & 200 & 1024 & $332.7$ ($4.4$) & $332.1$ ($3.2$) ${}^{~~~}$ & $321.5$ ($2.4$) ${}^{***}$ & [$305.2$ ($4.2$)] \\
        & 200 & 2048 & $327.2$ ($1.9$) & $326.5$ ($1.8$) ${}^{~~~}$ & $324.0$ ($4.5$) ${}^{*~~}$ & [$290.5$ ($1.4$)] \\
        \bottomrule
    \end{tabular}}%
    \par
    \vspace*{4ex}
    \caption{Test mean per-feature reconstruction errors by VAEs.}
    \label{tab:result3}
    \setlength{\tabcolsep}{8pt}
    {\small\begin{tabular}{ccccccc}
        \toprule
        \multicolumn{3}{c}{Setting} & \multicolumn{4}{c}{Results} \\
        \cmidrule(lr){1-3}\cmidrule(lr){4-7}
        Dataset & $\dim(\bm{z})$ & $\dim(\text{MLP})$ & L2 only & L2 + out-layer & \textbf{L2 + proposed} & dedicated decoder \\
        \midrule\midrule
        \textsc{Toy} & 4 & 32 & $0.758$ ($0.148$) & $0.412$ ($0.076$) ${}^{***}$ & $0.407$ ($0.068$) ${}^{***}$ & $0.806$ ($0.370$) \\ 
        \midrule
        \multirow{9}*{\textsc{Plant}}
        & 4 & 32 & $8.776$ ($0.183$) & --- & $8.633$ ($0.112$) ${}^{*~~~}$ & --- \\
        & 4 & 64 & $8.728$ ($0.166$) & --- & $8.672$ ($0.101$) ${}^{~~~~}$ & --- \\
        & 4 & 128 & $8.778$ ($0.157$) & --- & $8.679$ ($0.104$) ${}^{*~~~}$ & --- \\
        & 7 & 32 & $8.536$ ($0.232$) & --- & $8.322$ ($0.165$) ${}^{***~}$ & --- \\
        & 7 & 64 & $8.213$ ($0.154$) & --- & $8.061$ ($0.163$) ${}^{**~~}$ & --- \\
        & 7 & 128 & $8.206$ ($0.073$) & --- & $8.157$ ($0.145$) ${}^{~~~~}$ & --- \\
        & 11 & 32 & $8.468$ ($0.269$) & --- & $8.374$ ($0.294$) ${}^{~~~~}$ & --- \\
        & 11 & 64 & $8.106$ ($0.233$) & --- & $7.949$ ($0.166$) ${}^{**~~}$ & --- \\
        & 11 & 128 & $7.673$ ($0.155$) & --- & $7.567$ ($0.186$) ${}^{**~~}$ & --- \\
        \midrule
        \multirow{9}*{\textsc{Solar}}
        & 4 & 128 & $4.507$ ($0.963$) & $4.315$ ($0.463$) ${}^{~~~}$ & $2.574$ ($0.330$) ${}^{***}$ & --- \\
        & 4 & 256 & $4.507$ ($0.963$) & $4.315$ ($0.463$) ${}^{~~~}$ & $2.574$ ($0.330$) ${}^{***}$ & --- \\
        & 4 & 512 & $3.739$ ($0.337$) & $3.781$ ($0.462$) ${}^{~~~}$ & $2.757$ ($0.334$) ${}^{***}$ & --- \\
        & 13 & 128 & $3.195$ ($1.413$) & $2.953$ ($0.501$) ${}^{~~~}$ & $2.131$ ($0.156$) ${}^{*~~~}$ & --- \\
        & 13 & 256 & $3.215$ ($0.539$) & $3.145$ ($0.470$) ${}^{~~~}$ & $1.972$ ($0.164$) ${}^{***}$ & --- \\
        & 13 & 512 & $3.388$ ($0.319$) & $3.419$ ($0.218$) ${}^{~~~}$ & $2.228$ ($0.250$) ${}^{***}$ & --- \\
        & 54 & 128 & $2.367$ ($1.545$) & $1.873$ ($0.106$) ${}^{~~~}$ & $1.931$ ($0.075$) ${}^{~~~~}$ & --- \\
        & 54 & 256 & $2.732$ ($0.425$) & $2.710$ ($0.487$) ${}^{~~~}$ & $1.904$ ($0.157$) ${}^{***}$ & --- \\
        & 54 & 512 & $4.462$ ($0.216$) & $4.405$ ($0.332$) ${}^{~~~}$ & $3.033$ ($0.310$) ${}^{***}$ & --- \\
        \bottomrule
        \multicolumn{7}{r}{{\scriptsize Significant difference from the case of ``L2 only'' with \ \ ${}^{***}$: $p<0.001$, ~ ${}^{**}$: $p<0.01$, ~ or ~ ${}^*$: $p<0.05$.}}
    \end{tabular}}%
\end{table*}


\end{document}